\documentclass[preprint]{article}

\PassOptionsToPackage{numbers}{natbib}
\usepackage{neurips_2026}
\usepackage[utf8]{inputenc}
\usepackage[T1]{fontenc}
\usepackage{hyperref}
\usepackage{url}
\usepackage{microtype}
\usepackage{xcolor}
\usepackage{amsmath,amssymb,amsfonts}
\usepackage{amsthm}
\usepackage{booktabs}
\usepackage{array}
\usepackage{graphicx}
\usepackage{enumitem}
\usepackage{tikz}
\usetikzlibrary{positioning,calc,arrows.meta}
\newtheorem{proposition}{Proposition}

\graphicspath{{data/figures/}}

\title{Bounded Precision-Geometry Scaling for Robust Multi-Task Learning under Loss Scale Mismatch}
\author{%
  Krishna Subedi \\
  \texttt{krishna.subedi@neryva.com} \\
}
\date{}

\begin{document}
\maketitle

\begin{abstract}
Multi-task learning often combines losses that span several orders of magnitude, causing homoscedastic uncertainty weighting to degrade severely. We propose Bounded Precision-Geometry Scaling (BPGS), a method that maps each task's log-variance through a bounded sigmoid parameterisation anchored to detached batch loss statistics, and decouples network optimisation from uncertainty optimisation. Its normalised task weights are provably invariant to uniform rescaling under non-degenerate loss scales. We evaluate BPGS on synthetic stress tests and three real-world benchmarks: NYUv2 dense prediction, Yeast multi-label classification, and RF1 multi-target regression. Under pure loss rescaling from $\times 1$ to $\times 1000$, its macro score changes from 0.777 to 0.778, whereas Kendall weighting drops from 0.780 to 0.637; $\ell_1$-normalising Kendall's weights does not close the gap. On NYUv2, BPGS records the lowest depth absolute relative error (0.223), depth RMSE (0.790), and total loss (1.891) among all compared methods, including Nash-MTL. Sensitivity studies on batch size and calibration show small variation across the tested ranges, and runtime overhead relative to Kendall is under 1\%. BPGS posts the highest Yeast micro-F1 (0.616) and is competitive on RF1, though PCGrad leads RMSE and MAE there. These findings establish BPGS as a scale-robust alternative to homoscedastic uncertainty weighting, notably effective when loss-scale disparities dominate multi-task optimization.
\end{abstract}

\section{Introduction}
In multi-task learning, raw losses often differ in magnitude; for example, a pixel-level cross-entropy may dominate a regression residual by several orders of magnitude. If the weighting mechanism scales with loss magnitude, tasks with larger losses dominate the shared gradients, diminishing the updates for smaller-loss tasks, preventing effective training, and destabilising learned task-weighting parameters \cite{caruana1997multitask}.

Homoscedastic uncertainty weighting addresses this disparity by learning one log-variance per task \cite{kendall2018}. This works well when the losses have similar scales. However, under severe scale disparity, its unconstrained parameters can translate the scale gap into large precision shifts, which can destabilise optimisation. PCGrad \cite{yu2020pcgrad}, a gradient surgery method, tries to resolve directional conflicts among task gradients but leaves the scalar loss weighting problem unaddressed.

To address the gap left by the aforementioned methods, we propose Bounded Precision-Geometry Scaling (BPGS). BPGS confines every task's learned uncertainty to a finite interval by mapping each latent coordinate through a sigmoid and re-centring the output on detached batch log-loss statistics. A split objective keeps weight adaptation and network training on separate computational paths. We study whether this bounded, batch-anchored parameterisation can remain robust to severe loss-scale disparities without hurting performance on standard multi-task benchmarks.

We make four contributions:
\begin{enumerate}[leftmargin=1.25em]
\item We introduce BPGS as a bounded, batch-aware, split uncertainty-weighting method and prove that its normalised task weights are strictly invariant to uniform loss rescaling under non-degenerate loss scales.
\item Across three orders of magnitude of loss scaling ($\times1$--$\times1000$), BPGS maintains stable performance, whereas standard homoscedastic weighting \cite{kendall2018} degrades markedly; adding an $\ell_1$ normalisation constraint does not recover this stability.
\item On the NYUv2 dataset, BPGS outperforms competitive baselines, including Nash-MTL \cite{navon2022nash}, in depth estimation accuracy and aggregate multi-task loss. Detailed ablations isolate the empirical contributions of the batch-anchored parameterisation, first-batch calibration, and decoupled stop-gradient objective.
\item We systematically evaluate sensitivity to batch size, calibration parameters, and runtime overhead on NYUv2, and demonstrate that the performance gains of BPGS extend to multi-task tabular benchmarks (Yeast and RF1).
\end{enumerate}
\section{Related Work}
Multi-task networks typically optimise shared representations across multiple objectives \cite{caruana1997multitask} and often rely on adaptive weighting schemes to balance task contributions. BPGS builds directly on homoscedastic uncertainty weighting \cite{kendall2018}, which assigns every task a learned log-variance that simultaneously scales the task loss and regularises the uncertainty estimate. Other adaptive methods reweight tasks based on training dynamics: GradNorm \cite{chen2018gradnorm} balances gradient magnitudes to synchronise learning progress, while Dynamic Task Prioritization \cite{guo2018dtp} emphasises lagging objectives. In contrast, BPGS reparameterises the learned variance through a bounded mapping anchored to detached batch loss statistics. The design specifically targets loss-scale robustness rather than directional gradient conflict.

Gradient manipulation methods resolve task conflict directly in gradient space rather than in the loss domain. Sener and Koltun \cite{sener2018multiobjective} frame multi-task learning as multi-objective optimisation and adapt the Multiple Gradient Descent Algorithm (MGDA) \cite{desideri2012mgda} to converge toward Pareto-stationary solutions. Extending this paradigm, gradient surgery techniques explicitly alter the update direction to minimise negative transfer: PCGrad \cite{yu2020pcgrad} orthogonally projects conflicting task gradients, while CAGrad \cite{liu2021cagrad} regularises the update toward conflict-averse directions. Game-theoretic and accelerated multi-objective formulations such as Nash-MTL \cite{navon2022nash}, IMTL-G \cite{liu2021imtl}, and FAMO \cite{liu2023famo} similarly operate in gradient space by solving per-step optimisation subproblems. In contrast, BPGS operates in the loss domain through scalar uncertainty reparameterisation, neither modifying gradient geometry nor adding iterative optimisation overhead. Our benchmarks compare BPGS against standard loss-weighting methods alongside representative gradient-level methods, including Nash-MTL on NYUv2 and PCGrad on tabular tasks.

The resulting comparison set spans both paradigms, scoped to the scale-robustness question rather than a comprehensive multi-task optimisation survey.

\section{Method}

\begin{figure}[htbp]\centering
\resizebox{\textwidth}{!}{
\definecolor{netblue}{RGB}{31,119,180}
\definecolor{uncorange}{RGB}{216,100,20}
\definecolor{calteal}{RGB}{22,155,122}
\definecolor{fwd}{RGB}{65,65,65}

\tikzset{
  block/.style={draw=gray!65,fill=gray!7,rounded corners=2.5pt,align=center,inner sep=4pt,line width=0.7pt},
  oblock/.style={draw=netblue!85,fill=netblue!6,rounded corners=2.5pt,align=center,inner sep=4pt,line width=0.95pt},
  ublock/.style={draw=uncorange!90,fill=uncorange!6,rounded corners=2.5pt,align=center,inner sep=4pt,line width=0.95pt},
  lossnode/.style={circle,draw=gray!65,fill=white,minimum size=7.4mm,inner sep=0pt,line width=0.7pt,font=\footnotesize},
  sgnode/.style={circle,draw=gray!60,fill=gray!20,minimum size=7mm,inner sep=0pt,font=\footnotesize\bfseries,line width=0.7pt},
  sgbadge/.style={rectangle,rounded corners=1.5pt,draw=gray!55,fill=gray!18,inner sep=2.2pt,font=\tiny\bfseries,line width=0.5pt,text=black},
  smallhead/.style={draw=gray!65,fill=gray!8,rounded corners=1.5pt,minimum width=9.5mm,minimum height=5.6mm,inner sep=0pt,font=\footnotesize,line width=0.7pt},
  flow/.style={-{Latex[length=2.1mm,width=1.9mm]},line width=0.85pt,fwd},
  gradw/.style={-{Latex[length=2.8mm,width=2.5mm]},line width=1.3pt,netblue},
  gradth/.style={-{Latex[length=2.8mm,width=2.5mm]},line width=1.3pt,uncorange},
  dashcal/.style={-{Latex[length=2mm,width=1.8mm]},dashed,line width=0.9pt,calteal},
}

\begin{tikzpicture}[font=\footnotesize]

\filldraw[fill=netblue!5,draw=netblue!50,rounded corners=3pt,line width=0.9pt]
  (8.7,2.5) rectangle (13.5,6.7);

\node[block,minimum width=9mm,minimum height=9mm] (input) at (0.55,8.0) {$\mathbf{x}$};
\node[font=\scriptsize,gray] at (0.55,7.15) {input batch};

\node[block,minimum width=19mm,minimum height=19mm] (enc) at (2.6,8.0)
  {\textbf{shared encoder}\\[2pt] $f_{w}$\\[1pt] {\scriptsize parameters $w$}};

\node[smallhead] (h1) at (5.0,8.8) {$h_1$};
\node[smallhead] (h2) at (5.0,8.0) {$h_2$};
\node[font=\scriptsize,gray] at (5.0,7.15) {$\vdots$};
\node[smallhead] (hT) at (5.0,6.3) {$h_T$};
\node[font=\scriptsize,gray] at (5.0,9.35) {task heads};

\node[lossnode] (L1) at (6.8,8.8) {$L_1$};
\node[lossnode] (L2) at (6.8,8.0) {$L_2$};
\node[font=\scriptsize,gray] at (6.8,7.15) {$\vdots$};
\node[lossnode] (LT) at (6.8,6.3) {$L_T$};

\draw[flow] (input.east) -- (enc.west);
\draw[flow] (enc.east) -- (h1.west);
\draw[flow] (enc.east) -- (h2.west);
\draw[flow] (enc.east) -- (hT.west);
\draw[flow] (h1.east) -- (L1.west);
\draw[flow] (h2.east) -- (L2.west);
\draw[flow] (hT.east) -- (LT.west);

\draw[flow] (L1.east) -- (7.7,8.8);
\draw[flow] (L2.east) -- (7.7,8.0);
\draw[flow] (LT.east) -- (7.7,6.3);

\node[sgnode] (sgn) at (7.7,3.0) {sg};

\node[block,minimum width=42mm,minimum height=17mm] (stats) at (4.6,3.0)
  {\textbf{detached log-loss statistics}\\[2pt]
   {\scriptsize $\tilde L_i=\max(L_i,\varepsilon_{\log})$,~~ $\ell_i=\log\tilde L_i$}\\[1pt]
   {\scriptsize $\mu(L)=\mathrm{mean}_i\;\ell_i$}\\[1pt]
   {\scriptsize $\bar\varsigma(L)=\max(\mathrm{std}_i\;\ell_i,\;\varepsilon_{\mathrm{std}})$}};

\draw[flow] (sgn.west) -- (stats.east);
\draw[flow] (6.7,2.35) -- (8.7,2.72);
\node[font=\scriptsize] at (8.2,2.2) {$\mu(L),\bar\varsigma(L)$};

\node[font=\small\bfseries,text=netblue!75!black] at (11.1,6.32) {Bounded log-variance chart};
\node at (11.1,5.72) {$z_i=\tau_T\big(2\sigma(\theta_i)-1\big)$};
\node at (11.1,5.20) {$s_i=\mu(L)+\bar\varsigma(L)\,z_i$};
\node at (11.1,4.70) {$\tau_T=\sqrt{T-1}+0.1,\quad z_i\in(-\tau_T,\tau_T)$};

\fill[netblue!13] (9.6,3.08) rectangle (12.7,3.62);
\draw[dashed,netblue!55,line width=0.6pt] (9.6,2.98)  -- (9.6,3.72);
\draw[dashed,netblue!55,line width=0.6pt] (12.7,2.98) -- (12.7,3.72);
\draw[netblue!70,line width=0.8pt] (11.15,3.08) -- (11.15,3.62);
\draw[-{Latex[length=1.8mm]},gray!75,line width=0.7pt] (9.0,3.35) -- (13.3,3.35);

\node[font=\scriptsize] at (9.6,2.82)  {$\mu-\bar\varsigma\tau_T$};
\node[font=\scriptsize] at (11.15,2.82) {$\mu(L)$};
\node[font=\scriptsize] at (12.7,2.82) {$\mu+\bar\varsigma\tau_T$};

\fill[fwd] (10.1,3.35)  circle (0.07);
\fill[fwd] (11.6,3.35)  circle (0.07);
\fill[fwd] (12.38,3.35) circle (0.07);
\node[font=\scriptsize] at (10.1,3.86)  {$s_1$};
\node[font=\scriptsize] at (11.6,3.86)  {$s_2$};
\node[font=\scriptsize] at (12.38,3.86) {$s_T$};

\node[block,minimum width=46mm,minimum height=9.5mm] (theta) at (10.3,1.0)
  {\textbf{uncertainty coordinates}\\[-1pt]
   {\scriptsize $\theta_1,\dots,\theta_T\in\mathbb{R}$ (one per task)}};

\draw[flow] (10.3,1.475) -- (10.3,2.5);
\node[font=\scriptsize,anchor=west] at (10.48,1.95) {$z_i$};

\node[draw=calteal,dashed,fill=calteal!5,rounded corners=2pt,line width=0.9pt,
      align=center,inner sep=3.5pt,font=\scriptsize] (calib) at (10.7,-0.6)
  {\textbf{first-batch auto-calibration (one-shot)}\\[2pt]
   $z_i^{(0)}=\big(\mathrm{sg}[\log\tilde L_i]-\mu(L)\big)/\bar\varsigma(L)$\\[1pt]
   $\theta_i\leftarrow\sigma^{-1}\big((z_i^{(0)}/\tau_T+1)/2\big)$};

\draw[dashcal] (10.3,0.18) -- (10.3,0.525);
\node[font=\tiny,text=calteal,anchor=west] at (10.46,0.3) {init};

\node[block,minimum width=40mm,minimum height=15mm] (prec) at (16.0,5.1)
  {\textbf{precision \& L1-normalization}\\[2pt]
   {\scriptsize $\omega_i=e^{-s_i}$}\\[1pt]
   {\scriptsize $\alpha_i=\omega_i\big/\sum_{j=1}^{T}\omega_j$}};

\draw[flow] (13.5,5.1) -- (prec.west);
\node[font=\scriptsize] at (13.8,5.33) {$s_i$};

\node[oblock,minimum width=48mm,minimum height=14mm] (Jnet) at (17.9,8.0)
  {\textbf{network objective}\\[2pt]
   {\footnotesize $J_{\mathrm{net}}=\sum_{i=1}^{T}\mathrm{sg}[\alpha_i]\,L_i$}};

\node[ublock,minimum width=56mm,minimum height=12mm] (Junc) at (16.9,1.0)
  {\textbf{uncertainty objective}\\[2pt]
   {\footnotesize $J_{\mathrm{unc}}=\sum_{i=1}^{T}\big[\tfrac{1}{2}\,\omega_i\,\mathrm{sg}[L_i]+\tfrac{1}{2}\,s_i\big]$}};

\draw[flow] (7.7,8.0) -- (Jnet.west);
\node[font=\scriptsize] at (11.6,8.28) {current-batch task losses $L_1,\dots,L_T$};

\draw[flow] (16.55,5.85) .. controls (16.75,6.55) and (17.05,6.85) .. (17.6,7.3);
\node[sgbadge] at (16.68,6.35) {sg};
\node[font=\scriptsize] at (16.02,6.62) {$\alpha_i$};

\draw[flow] (12.9,2.5) -- (15.0,1.6);
\node[font=\scriptsize] at (13.6,2.38) {$s_i$};
\draw[flow] (16.0,4.35) -- (16.0,1.6);
\node[font=\scriptsize,anchor=west] at (16.16,3.1) {$\omega_i$};

\draw[flow] (sgn.south) -- (7.7,-1.75) -| (Junc.south);
\node[font=\scriptsize] at (12.3,-1.5) {$\mathrm{sg}[L_i]$ — detached losses (every iteration)};

\draw[gradth] (Junc.west) -- (theta.east);
\node[font=\scriptsize,text=uncorange] at (13.35,1.34) {$g_\theta\,\nabla_\theta J_{\mathrm{unc}}$};
\node[font=\tiny,gray] at (13.35,0.68) {$(g_\theta=100)$};

\draw[gradw] (17.9,8.7) -- (17.9,9.7) -- (2.6,9.7) -- (2.6,8.98);
\node[font=\scriptsize,text=netblue] at (10.25,9.98)
  {network update: $\nabla_w J_{\mathrm{net}}=\sum_i \mathrm{sg}[\alpha_i]\,\nabla_w L_i$  ($\alpha_i$ detached)};

\node[font=\footnotesize\bfseries,gray] at (0.75,9.7) {BPGS};

\fill[fwd] (7.7,8.8) circle (0.055);
\fill[fwd] (7.7,8.0) circle (0.055);
\fill[fwd] (7.7,6.3) circle (0.055);
\draw[flow] (7.7,8.8) -- (sgn.north);

\draw[gradw] (0.5,-2.35) -- (1.4,-2.35);
\node[font=\scriptsize,anchor=west] at (1.6,-2.35) {network gradient path (updates $w$)};
\draw[gradth] (6.4,-2.35) -- (7.3,-2.35);
\node[font=\scriptsize,anchor=west] at (7.5,-2.35) {uncertainty path (updates $\theta$)};
\node[sgbadge] at (11.7,-2.35) {sg};
\node[font=\scriptsize,anchor=west] at (12.1,-2.35) {stop-gradient};
\draw[draw=calteal,dashed,fill=calteal!5,rounded corners=1pt,line width=0.8pt]
  (15.6,-2.55) rectangle (16.3,-2.15);
\node[font=\scriptsize,anchor=west] at (16.5,-2.35) {one-shot first-batch calibration};

\end{tikzpicture}}
\caption{Schematic overview of BPGS. The shared encoder $f_w$ yields $T$ task losses per batch.
A stop-gradient branch extracts log-loss statistics $\mu(L)$ and $\bar\varsigma(L)$, which anchor
a bounded chart that restricts each task's log-variance to a finite interval
(\S\ref{sec:boundedness}). Precisions derived from the chart are $\ell_1$-normalised into task weights
$\alpha_i$. Training is split: the network pass minimises a weighted loss with detached weights,
while a separate uncertainty pass adjusts the $\theta_i$ coordinates with detached losses.
First-batch auto-calibration sets the initial $\theta_i$ values; equations and design details
appear in \S\ref{sec:chart}--\ref{sec:boundedness}.}
\label{fig:bpgs_overview}
\end{figure}
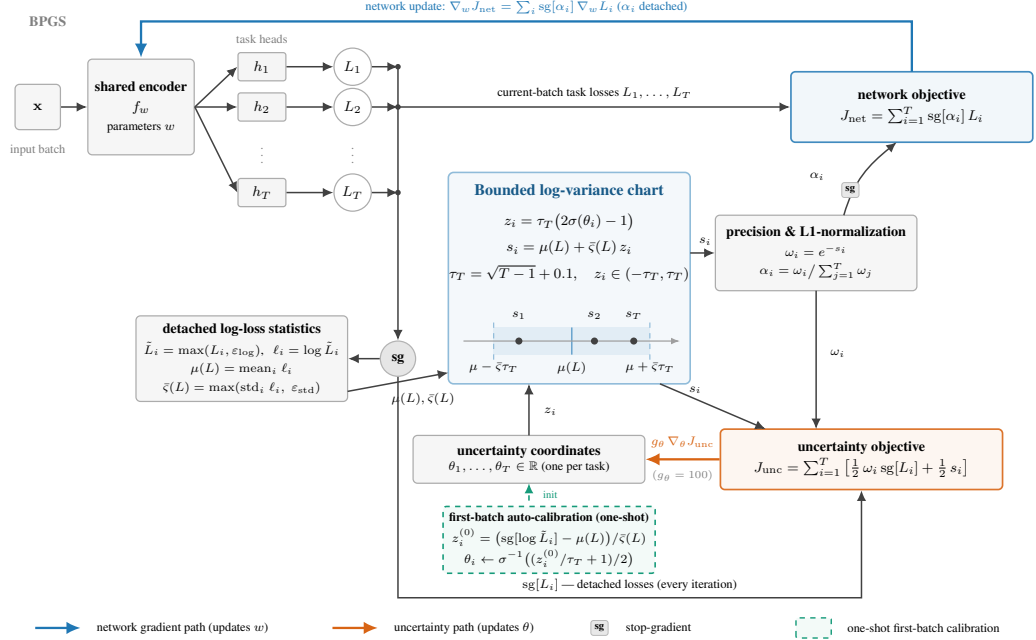

\subsection{Problem setup}
Let $T$ denote the number of tasks and $w$ the shared network parameters. The loss vector
\begin{equation}
L(w) = \bigl(L_1(w), \dots, L_T(w)\bigr)
\end{equation}
collects all raw task losses for the current batch. We seek adaptive task weights that stay informative even when these losses differ by orders of magnitude. BPGS approaches the problem by learning one latent uncertainty coordinate per task and mapping it into a bounded log-variance chart, as described next.

\subsection{Canonical BPGS chart}
\label{sec:chart}
The batch-aware form of BPGS begins by computing detached log-loss statistics from the current minibatch:
\begin{equation}
\widetilde L_i = \max(L_i, \varepsilon_{\log}), \qquad
\mu(L) = \frac{1}{T}\sum_{i=1}^{T}\operatorname{sg}[\log \widetilde L_i],
\end{equation}
and
\begin{equation}
\bar{\varsigma}(L) =
\max\!\left(
\sqrt{\frac{1}{T}\sum_{i=1}^{T}\bigl(\operatorname{sg}[\log \widetilde L_i] - \mu(L)\bigr)^2},
\varepsilon_{\mathrm{std}}
\right),
\end{equation}
where $\varepsilon_{\log}, \varepsilon_{\mathrm{std}} > 0$ are small numerical stability floors preventing zero-argument logarithms and degenerate zero variance.
Each task owns a single learnable uncertainty coordinate $\theta_i \in \mathbb{R}$. The latent radius depends on the task count:
\begin{equation}
\tau_T = \sqrt{T-1}+0.1,
\label{eq:tau}
\end{equation}
and the bounded latent coordinate is
\begin{equation}
z_i(\theta_i) = \tau_T \bigl(2\sigma(\theta_i)-1\bigr),
\end{equation}
where $\sigma(\cdot)$ is the logistic sigmoid. Combining $z_i$ with the batch statistics gives the canonical BPGS log-variance:
\begin{equation}
s_i(\theta_i;L) = \mu(L) + \bar{\varsigma}(L)\, z_i(\theta_i).
\end{equation}

The radius in Eq.~\eqref{eq:tau} follows from the geometry of a population-standardised
$T$-vector. In the unclipped regime ($\bar{\varsigma}(L) > \varepsilon_{\mathrm{std}}$), its coordinates have zero mean and unit average
square. If one coordinate has magnitude $M$, the other $T-1$ coordinates must sum to $-M$;
their squared sum is minimised when they are equal. Hence
$T \ge M^2+(T-1)(M/(T-1))^2$, giving $M\le\sqrt{T-1}$. The added $0.1$ is a numerical margin
that keeps first-batch inverse-sigmoid calibration away from the chart boundary.
From $s_i$ we obtain the task precision
\begin{equation}
\omega_i(\theta_i;L) = e^{-s_i(\theta_i;L)},
\end{equation}
and the network weight by $\ell_1$ normalisation:
\begin{equation}
\alpha_i(\theta_i;L) =
\frac{\omega_i(\theta_i;L)}{\sum_{j=1}^{T}\omega_j(\theta_j;L)}.
\end{equation}

The net effect: every task's log-variance lives in a finite, batch-anchored interval rather than on the entire real line.

Before the first network update, BPGS runs a one-shot auto-calibration step. It maps the detached log-loss statistics of the initial batch into the bounded chart to set the starting $\theta_i$ values. The ablation study (Section~\ref{sec:ablation}) tests this initialisation rule separately from the batch-aware chart.

\subsection{Split optimisation objectives}
Network parameters and uncertainty coordinates are trained under separate objectives. The network loss is
\begin{equation}
J_{\mathrm{net}}(w \mid \theta, L)
=
\sum_{i=1}^{T} \operatorname{sg}\!\left[\alpha_i(\theta_i;L)\right] L_i,
\end{equation}
where $\operatorname{sg}[\cdot]$ denotes stop-gradient. The uncertainty objective is
\begin{equation}
J_{\mathrm{unc}}(\theta \mid L)
=
\sum_{i=1}^{T}
\left[
\frac{1}{2}\omega_i(\theta_i;L)\operatorname{sg}[L_i]
+
\frac{1}{2}s_i(\theta_i;L)
\right].
\label{eq:unc}
\end{equation}
While Eq.~\eqref{eq:unc} shares its loss form with Kendall et al.~\cite{kendall2018}, optimization is performed with respect to the latent coordinate $\theta_i$ rather than an unconstrained log-variance. Confining $s_i(\theta_i; L)$ to a batch-conditional interval prevents unbounded parameter drift and fundamentally stabilizes the optimization dynamics.

A fixed gradient scale $g_\theta=100$ is applied to the uncertainty update across all reported
configurations. It compensates for the smaller magnitudes of the uncertainty gradients and
affects only update speed, not the weighting rule itself (full optimiser settings:
Appendix~\ref{app:reproducibility}).

Both passes draw on the same current-batch losses at every iteration. In $J_{\mathrm{net}}$, the
normalised weights are detached so that weight changes do not reshape the network gradient.
Symmetrically, $J_{\mathrm{unc}}$ detaches the losses and batch statistics. This separation
prevents the two updates from interfering within a single step.

\subsection{Batch-conditional boundedness}
\label{sec:boundedness}
The formal property underpinning the robustness argument is that $s_i$ is bounded conditionally on the current batch. Fix a batch $L$ and take any finite $\theta_i$:
\begin{equation}
\mu(L)-\bar{\varsigma}(L)\tau_T
<
s_i(\theta_i;L)
<
\mu(L)+\bar{\varsigma}(L)\tau_T.
\end{equation}
The bound follows from $\sigma(\theta_i)\in(0,1)$, which forces
\begin{equation}
z_i(\theta_i)\in(-\tau_T,\tau_T).
\end{equation}
Substituting into the affine map for $s_i$ yields the stated interval. As a consequence, the
induced precision $\omega_i(\theta_i;L)$ is strictly positive and bounded for every fixed batch.
Note the qualifier: the guarantee is batch-conditional, not globally independent of the observed
loss statistics.

Boundedness alone does not ensure learnability near the boundary. As $z_i$ approaches
$\pm\tau_T$, the sigmoid derivative vanishes and recovery may stall. We inspected the logged
trajectories from all main-paper runs (Appendix~\ref{app:chart_behavior}). The most extreme
ratio observed was $|z_i|/\tau_T=0.93$, occurring at initialisation on NYUv2; by the final
epoch, the caps settled to at most 0.84 (NYUv2), 0.69 (Yeast), and 0.89 (RF1). Even the
closest point to the boundary retained a derivative factor of roughly $0.033$---about
13\% of its maximum---and final-epoch values reached approximately $0.074$. No logged run
exhibited boundary locking. This margin is empirical, not a formal guarantee for arbitrary
tasks or training regimes.

\subsection{Uniform rescaling invariance}
\label{sec:invariance}
\begin{proposition}[Uniform rescaling invariance]
\label{prop:invariance}
Let $L'_i=cL_i$ for all tasks $i \in \{1,\dots,T\}$ with scalar multiplier $c>0$. For any loss vector satisfying $L_i > \varepsilon_{\log}$ and $\bar{\varsigma}(L) > \varepsilon_{\mathrm{std}}$ on both batches, the normalised task weights satisfy $\alpha_i(\theta_i;L')=\alpha_i(\theta_i;L)$ for all $i$. First-batch auto-calibration also produces identical initial $\theta_i$ values under the same rescaling.
\end{proposition}

\begin{proof}
Uniform rescaling adds $\log c$ to every log-loss. Therefore $\mu(L')=\mu(L)+\log c$ while
$\bar{\varsigma}(L')=\bar{\varsigma}(L)$, and $z_i(\theta_i)$ is unchanged. Hence
$s_i(\theta_i;L')=s_i(\theta_i;L)+\log c$, so each raw precision is multiplied by the same
factor $1/c$, which cancels in the $\ell_1$ normalisation. The standardised coordinates used for
first-batch calibration are unchanged by the same additive shift.
\end{proof}

This result is about normalised weights on a single batch; it says nothing about full optimiser
trajectories or final predictive metrics. The pure-rescaling experiment (Section~\ref{sec:rescaling_results})
probes that broader empirical behaviour.

\section{Experimental Setup}
\subsection{Benchmarks}
We evaluate BPGS across different benchmark and stress settings. The main dense-prediction comparison uses the NYUv2 multi-task benchmark \cite{silberman2012nyuv2}, which contains 795 training images and 654 validation images with three tasks: semantic segmentation, depth estimation, and surface-normal prediction. This comparison runs for 120 epochs and reports final-epoch metrics aggregated over seeds.

A separate NYUv2 ablation study uses a fixed 50\% training subset with the validation split unchanged. It compares four BPGS variants against Kendall uncertainty weighting to isolate the bounded chart and the initialisation rule: the canonical batch-aware auto-calibrated form, a batch-aware fixed-initialisation form, a stateless fixed-initialisation form, and a stateless auto-calibrated form. The ablation runs for 60 epochs over the same three training seeds.

Two real-data benchmarks evaluate the generalisability of the gains beyond dense prediction. Yeast \cite{elisseeff2001yeast} is a 14-label multi-label classification task with 1,500 training and 917 validation examples. RF1 \cite{spyromitros2016mtr} is an 8-target regression task with 4,108 training and 5,017 validation examples. Both report final task metrics aggregated over three seeds from runs of up to 120 epochs.

Four controlled studies reuse the same NYUv2 50\% subset. The stop-gradient study pairs canonical
BPGS against a coupled variant. The batch-size study sweeps sizes 4, 8, and 16. Both run for 60
epochs over three seeds. A first-batch study varies only the loader seed (1001, 2001, 3001) while
fixing the model seed, isolating calibration sensitivity. The overhead study records per-epoch
wall-clock time and peak CUDA memory for BPGS versus Kendall over 60 epochs and three seeds.

Three synthetic studies test scale robustness. The pure loss-rescaling study applies post-hoc multipliers $\times 1$, $\times 10$, $\times 100$, and $\times 1000$ to check invariance to raw loss magnitude. The scale-stress study uses the same grid in a synthetic multi-task setting where increasing scale imbalance can degrade all methods. The heterogeneous mixed-stress study adds three regimes---\emph{clean}, \emph{conflict}, and \emph{noisy}---to test what happens when tasks differ not only in scale but also in interaction structure.

A separate diagnostic asks whether normalisation alone explains the rescaling result: it compares
BPGS, Kendall, and Kendall with $\ell_1$-normalised precision weights on the same rescaling grid. This
run uses seeds 42, 123, and 999 and is reported apart from the main rescaling table.

\subsection{Compared methods and metrics}
The comparison set varies by benchmark. On NYUv2, the main comparison includes BPGS, Static, Kendall uncertainty weighting \cite{kendall2018}, UWSO, and Nash-MTL \cite{navon2022nash}. The ablation restricts to the four BPGS variants against Kendall. Yeast and RF1 compare BPGS against Static, Kendall, PCGrad \cite{yu2020pcgrad}, a GradNorm-inspired gradient-norm proxy \cite{chen2018gradnorm}, and UWSO. The pure loss-rescaling and scale-stress studies compare BPGS, Kendall, and UWSO. The heterogeneous mixed-stress study uses the same three plus PCGrad.

For NYUv2, we report segmentation mIoU, depth absolute relative error, depth RMSE, mean angular error for normals, normal accuracy within $11.25^\circ$, total validation loss, and in the main benchmark also $\Delta_M$ against the Static baseline \cite{liu2019mtan}. Yeast metrics are macro-F1, micro-F1, and Hamming accuracy. RF1 metrics are $R^2$, RMSE, and MAE. The synthetic stress studies use macro score as the primary statistic, with worst-task score added for the heterogeneous analysis and stress diagnostics.

\subsection{Reporting protocol}
Unless stated otherwise, values are means $\pm$ standard deviations across seeds. The main NYUv2,
Yeast, RF1, rescaling, and heterogeneous-stress results use seeds 42, 43, and 44. The synthetic
scale-stress study in Table~\ref{tab:stress_scale} uses 10 seeds. Synthetic stress results aggregate per-seed summaries,
whereas NYUv2, Yeast, and RF1 tables use the final recorded epoch. The supplementary
materials contain the code, configurations, and reproduction instructions.

Because checkpoint summaries are not uniformly available across methods, mixing best- and
final-checkpoint numbers would introduce method-dependent extraction rules. We therefore use final
epoch metrics for all paper tables. Final NYUv2 runs select checkpoints by \texttt{val/miou} in max
mode, while the ablation and controlled studies use the dataset default
\texttt{val/total\_loss} in min mode. The uncertainty update uses $g_\theta=100$ and full settings are
in Appendix~\ref{app:reproducibility}.

\section{Results}
\subsection{Robustness to loss-scale mismatch}
\label{sec:rescaling_results}
Figure~\ref{fig:stress_summary} separates two perturbation types as the right column isolates pure post-hoc loss rescaling, where only the numerical scale changes. In that setting, BPGS is near-invariant: its macro score stays at 0.777, 0.784, 0.781, and 0.778 from $\times 1$ through $\times 1000$ (Table~\ref{tab:stress_rescaling} and Appendix Figure~\ref{fig:appendix_rescaling}). Proposition~\ref{prop:invariance} shows that the normalised weights are invariant under this rescaling in the non-degenerate regime. The table verifies that this invariance carries through to training outcomes. Kendall-style uncertainty weighting starts slightly higher at $\times 1$, but drops monotonically to 0.637 at $\times 1000$, while UWSO improves from a lower $\times 1$ baseline.

\begin{table}[t]
\centering
\caption{Pure loss rescaling: Macro Score across scale factors. Mean $\pm$ std over 3 seeds.}
\label{tab:stress_rescaling}
\resizebox{\columnwidth}{!}{%
\begin{tabular}{l c c c c}
\toprule
Method & $\times{1}$ & $\times{10}$ & $\times{100}$ & $\times{1000}$ \\
\midrule
Kendall & $\mathbf{0.780 \pm 0.007}$ & $0.769 \pm 0.013$ & $0.714 \pm 0.017$ & $0.637 \pm 0.015$ \\
UWSO & $0.622 \pm 0.073$ & $0.674 \pm 0.036$ & $0.688 \pm 0.035$ & $0.681 \pm 0.014$ \\
BPGS & $0.777 \pm 0.012$ & $\mathbf{0.784 \pm 0.008}$ & $\mathbf{0.781 \pm 0.001}$ & $\mathbf{0.778 \pm 0.006}$ \\
\bottomrule
\end{tabular}
}
\end{table}

\paragraph{Normalisation is not sufficient.}\label{sec:norm_ablation}
On a separate three-seed diagnostic, $\ell_1$-normalised Kendall improves over unnormalised Kendall at
$\times1000$ (0.689 versus 0.658) but still falls by 0.105 from $\times1$ to $\times1000$.
BPGS falls by 0.004 (Appendix~\ref{app:norm_ablation}). This diagnostic uses a different
seed set from the main rescaling table, so we treat it as evidence that
normalisation alone is insufficient, not as a seed-matched estimate of the main-table gap.

\begin{figure}[htbp]
\centering
\includegraphics[width=\textwidth]{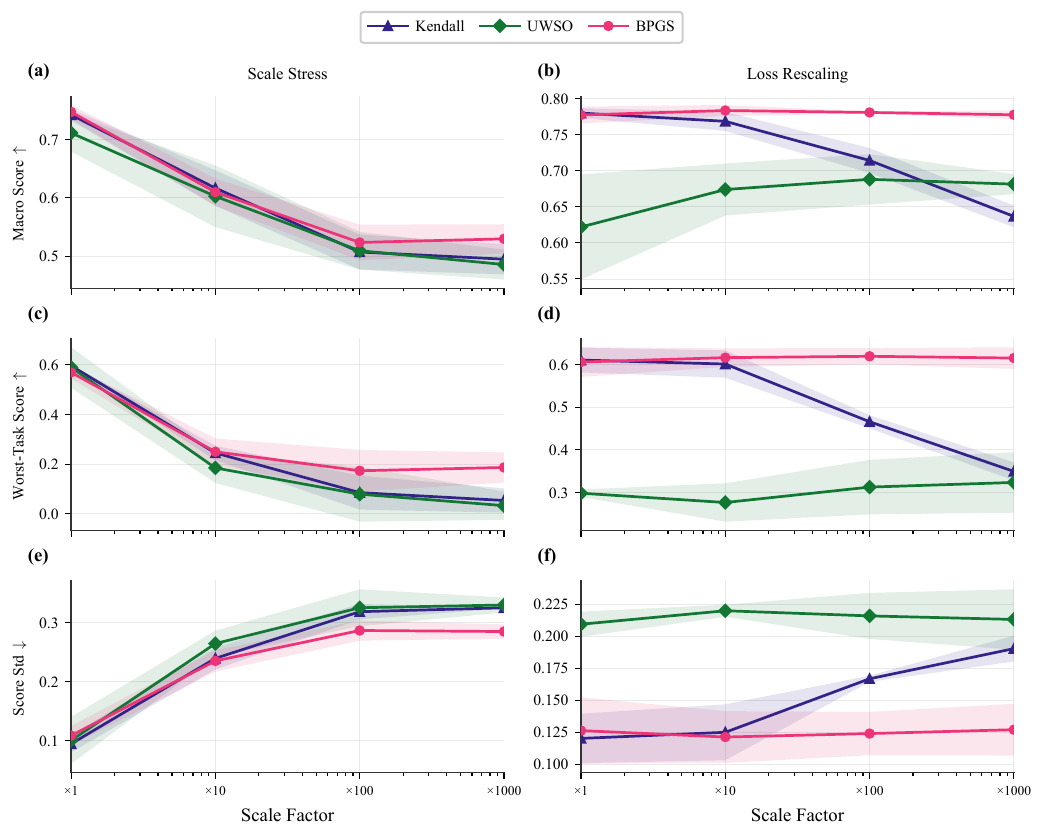}
\caption{Synthetic robustness diagnostics under scale mismatch. Left column: scale stress, where performance is evaluated as task scales become increasingly imbalanced. Right column: pure loss rescaling, where only the reported magnitude changes. BPGS remains nearly invariant under pure rescaling and retains the strongest macro score at the largest scale-stress factors.}
\label{fig:stress_summary}
\end{figure}

The left column of Figure~\ref{fig:stress_summary} is the difficult test as the perturbation changes the optimisation problem, not only its units. All three methods degrade as the factor increases, but BPGS leads at every tested multiplier, reaching 0.543 at $\times 1000$ versus 0.499 for Kendall and 0.494 for UWSO. The relative drop from $\times 1$ to $\times 1000$ is 27.0\% for BPGS, compared with 32.3\% for Kendall and 28.5\% for UWSO (Appendix Table~\ref{tab:degradation}; Appendix Figure~\ref{fig:appendix_scale_stress}). The BPGS--Kendall gap at $\times 1000$ is 0.044, or about 1.1 combined standard deviations, a comparison of means rather than a formal significance test.

The heterogeneous mixed-stress study is less favourable to BPGS than the scale-mismatch studies (Appendix Table~\ref{tab:stress_heterogeneous} and Appendix Figure~\ref{fig:appendix_heterogeneous}). Kendall and PCGrad attain higher macro scores in the noisy and conflict regimes, whereas BPGS gives the strongest worst-task score in the clean regime at 0.301 while UWSO collapses to 0 throughout. These results limit the rescaling advantage to scale-dominated settings and do not extend it to conflict-dominated ones.

\subsection{Main NYUv2 benchmark}
Table~\ref{tab:nyuv2_main} shows that no method dominates every metric. Static achieves the best
segmentation mIoU, and UWSO is strongest on both surface-normal metrics. BPGS leads on depth
and on the aggregate: best absolute relative error (0.223), depth RMSE (0.790),
total loss (1.891), and $\Delta_M$ (0.078). Appendix Figures~\ref{fig:appendix_nyuv2_bars}--\ref{fig:appendix_nyuv2_delta} visualise these per-metric comparisons.

\begin{table}[t]
\centering
\caption{NYUv2 dense prediction benchmark. All methods trained for 120 epochs. Final-epoch metrics reported as mean $\pm$ std over 3 seeds. $\Delta_M$ is computed against the Static baseline; the Nash-MTL value is derived from the reported mean metrics and is shown approximately.}
\label{tab:nyuv2_main}
\resizebox{\columnwidth}{!}{%
\begin{tabular}{l c c c c c c c}
\toprule
Method & mIoU $\uparrow$ & Abs Rel $\downarrow$ & RMSE $\downarrow$ & Angle $\downarrow$ & Within 11.25$^\circ$ $\uparrow$ & Total Loss $\downarrow$ & $\Delta_M$ $\uparrow$ \\
\midrule
Static & $\mathbf{0.341 \pm 0.006}$ & $0.239 \pm 0.000$ & $0.829 \pm 0.011$ & $35.934 \pm 2.055$ & $0.169 \pm 0.009$ & $1.992 \pm 0.033$ & $0.000 \pm 0.000$ \\
Kendall & $0.281 \pm 0.013$ & $0.229 \pm 0.001$ & $\underline{0.808 \pm 0.003}$ & $28.629 \pm 0.431$ & $0.231 \pm 0.007$ & $1.977 \pm 0.029$ & $0.022 \pm 0.000$ \\
UWSO & $0.294 \pm 0.007$ & $\underline{0.228 \pm 0.003}$ & $0.815 \pm 0.012$ & $\mathbf{26.110 \pm 0.383}$ & $\mathbf{0.271 \pm 0.007}$ & $\underline{1.936 \pm 0.023}$ & $\underline{0.059 \pm 0.000}$ \\
\textbf{BPGS} & $\underline{0.312 \pm 0.001}$ & $\mathbf{0.223 \pm 0.002}$ & $\mathbf{0.790 \pm 0.010}$ & $\underline{26.851 \pm 0.334}$ & $\underline{0.257 \pm 0.005}$ & $\mathbf{1.891 \pm 0.014}$ & $\mathbf{0.078 \pm 0.000}$ \\
Nash-MTL & $0.252 \pm 0.027$ & $0.242 \pm 0.009$ & $0.856 \pm 0.029$ & $30.180 \pm 1.010$ & $0.211 \pm 0.011$ & $2.071 \pm 0.070$ & $\approx -0.038$ \\
\bottomrule
\end{tabular}
}
\end{table}

Relative to the strongest competing values, BPGS improves absolute relative error from 0.228 under UWSO to 0.223 and RMSE from 0.808 under Kendall to 0.790, and lowers total loss from 1.936 to 1.891. The NYUv2 comparison therefore shows that the bounded parameterisation holds its own on a standard dense-prediction benchmark while retaining its scale-robustness properties.

Nash-MTL has the weakest mean mIoU (0.252) and largest mIoU standard deviation (0.027), though its
within-$11.25^\circ$ accuracy (0.211) exceeds Static (0.169). It serves as a recent
multi-objective reference point under the same protocol. The single-benchmark comparison does not permit broader conclusions.

\subsection{Ablation on chart and initialisation}
\label{sec:ablation}
Table~\ref{tab:ablation} isolates the batch-aware chart from initialisation. The stateless
auto-calibrated variant collapses (mIoU 0.058; total loss 6.142): without the batch-aware chart,
first-batch calibration alone cannot sustain training. Appendix
Figure~\ref{fig:appendix_ablation_val_curves} traces this collapse in the validation curves, and
Appendix Figure~\ref{fig:appendix_ablation_weights} contrasts the task-weight trajectories.

\begin{table}[t]
\centering
\caption{Ablation study: BPGS variants vs.\ Kendall uncertainty weighting on NYUv2. All methods trained for 60 epochs. Mean $\pm$ std over 3 seeds.}
\label{tab:ablation}
\resizebox{\columnwidth}{!}{%
\begin{tabular}{l c c c c c c}
\toprule
Method & mIoU $\uparrow$ & Abs Rel $\downarrow$ & RMSE $\downarrow$ & Angle $\downarrow$ & Within 11.25$^\circ$ $\uparrow$ & Total Loss $\downarrow$ \\
\midrule
Kendall & $\underline{0.137 \pm 0.006}$ & $0.294 \pm 0.010$ & $\underline{0.990 \pm 0.016}$ & $38.681 \pm 0.936$ & $0.130 \pm 0.005$ & $2.546 \pm 0.035$ \\
BPGS (canonical) & $0.128 \pm 0.008$ & $\mathbf{0.283 \pm 0.005}$ & $\mathbf{0.988 \pm 0.043}$ & $\mathbf{33.834 \pm 1.068}$ & $\mathbf{0.161 \pm 0.015}$ & $\mathbf{2.488 \pm 0.075}$ \\
\textbf{BPGS (BA-fixed)} & $\mathbf{0.143 \pm 0.006}$ & $0.291 \pm 0.005$ & $0.994 \pm 0.022$ & $39.097 \pm 0.478$ & $0.130 \pm 0.004$ & $\underline{2.524 \pm 0.047}$ \\
BPGS (SL-auto) & $0.058 \pm 0.006$ & $0.331 \pm 0.003$ & $1.305 \pm 0.015$ & $\underline{34.661 \pm 1.310}$ & $\underline{0.147 \pm 0.024}$ & $6.142 \pm 1.332$ \\
BPGS (SL-fixed) & $0.120 \pm 0.007$ & $\underline{0.287 \pm 0.006}$ & $1.013 \pm 0.025$ & $35.183 \pm 0.963$ & $0.142 \pm 0.010$ & $2.566 \pm 0.051$ \\
\bottomrule
\end{tabular}
}
\end{table}

Among the stable variants, the best values split between the two batch-aware forms. The canonical batch-aware BPGS variant leads on five of six metrics: depth absolute relative error, depth RMSE, mean angular error, normal accuracy within $11.25^\circ$, and total loss. The batch-aware fixed-initialisation variant leads on mIoU. The stateless fixed variant never matches the canonical batch-aware variant on any metric and trails both batch-aware forms on mIoU and total loss. The batch-aware chart is therefore the critical component. The initialisation rule shifts which task benefits most, without destabilising the overall run.

A separate ablation isolates the split stop-gradient while keeping the bounded, batch-aware chart
and initialisation fixed (Appendix~\ref{app:sensitivity}). Removing it alters the uncertainty
trajectory: $\theta_{\max}$ grows from 3.11 to 3.99 instead of decaying from 3.08 to
2.44, and its final seed-to-seed standard deviation rises from 0.014 to 0.192. Segmentation
mIoU drops from 0.198 to 0.190. The split is therefore a substantive design choice, not an
implementation detail.

\subsection{Sensitivity to batch statistics and calibration}
The batch-aware chart was stable over the tested batch sizes. Across batch sizes 4, 8, and 16,
seed-averaged validation metrics have coefficients of variation below 3.5\%, with most below
1.2\%; final $\theta_{\max}$ ranges from 2.40 to 2.59
(Appendix~\ref{app:sensitivity}). When only the first loader batch is changed, all reported
task metrics have coefficients of variation below 2\%, and the final $\theta_{\max}$ values differ by
at most 0.3\%. These studies bound the claim to the tested ranges; they do not establish
insensitivity to arbitrary batch statistics.

\subsection{Transfer to Yeast and RF1}
The real-data benchmarks separate methods less clearly than the synthetic stress tests, so we read them as consistency checks rather than ranking evidence. On Yeast, BPGS is best on all three reported metrics, with macro-F1 0.436, micro-F1 0.616, and Hamming accuracy 0.773 (Table~\ref{tab:real_data}; Appendix Figure~\ref{fig:appendix_yeast_curves}). On RF1, BPGS is not the top method: PCGrad gives the lowest RMSE and MAE, and UWSO gives the best mean $R^2$. BPGS nonetheless posts the second-best $R^2$ ($-0.429$) with RMSE and MAE within about 2\% of the best values (Appendix Figure~\ref{fig:appendix_rf1_curves}).

\begin{table}[t]
\centering
\caption{Real-data transfer benchmarks. Final-epoch task metrics reported as mean $\pm$ std over 3 seeds. Higher is better for Yeast metrics and RF1 $R^2$; lower is better for RF1 RMSE and MAE.}
\label{tab:real_data}
\resizebox{\columnwidth}{!}{%
\begin{tabular}{l c c c c c c}
\toprule
Method & Macro-F1 $\uparrow$ & Micro-F1 $\uparrow$ & Hamming Acc $\uparrow$ & $R^2$ $\uparrow$ & RMSE $\downarrow$ & MAE $\downarrow$ \\
\midrule
Static & $0.430 \pm 0.008$ & $\underline{0.610 \pm 0.009}$ & $0.772 \pm 0.003$ & $-0.453 \pm 0.127$ & $29.875 \pm 0.485$ & $23.376 \pm 0.451$ \\
Kendall & $\underline{0.431 \pm 0.003}$ & $0.609 \pm 0.001$ & $0.768 \pm 0.002$ & $-0.439 \pm 0.164$ & $\underline{29.801 \pm 0.821}$ & $\underline{23.219 \pm 0.790}$ \\
UWSO & $0.142 \pm 0.014$ & $0.486 \pm 0.008$ & $0.767 \pm 0.002$ & $\mathbf{-0.408 \pm 0.193}$ & $30.144 \pm 1.137$ & $23.449 \pm 1.230$ \\
PCGrad & $0.430 \pm 0.007$ & $0.604 \pm 0.010$ & $0.769 \pm 0.002$ & $-0.432 \pm 0.107$ & $\mathbf{29.596 \pm 0.385}$ & $\mathbf{23.098 \pm 0.368}$ \\
GradNormProxy & $0.408 \pm 0.028$ & $0.597 \pm 0.014$ & $\underline{0.772 \pm 0.001}$ & $-0.508 \pm 0.248$ & $30.033 \pm 0.482$ & $23.380 \pm 0.582$ \\
\textbf{BPGS} & $\mathbf{0.436 \pm 0.006}$ & $\mathbf{0.616 \pm 0.008}$ & $\mathbf{0.773 \pm 0.003}$ & $\underline{-0.429 \pm 0.213}$ & $30.158 \pm 0.952$ & $23.548 \pm 0.908$ \\
\bottomrule
\end{tabular}
}
\end{table}

\subsection{Computational overhead}
On the NYUv2 50\% subset, BPGS takes 40.979 seconds per epoch versus 40.916 for Kendall and uses
3368 MB peak GPU memory versus 3339 MB: increases of 0.15\% and 0.87\%, respectively
(Appendix~\ref{app:overhead}). The method adds three task-level parameters to an
18.9M-parameter network. At this scale, the batch statistics and separate
uncertainty update add no measurable efficiency penalty.

\section{Discussion and Limitations}
Collectively, the empirical results validate BPGS across its primary design objectives: controlled rescaling tests verify its invariance under extreme loss-scale disparities, while standard multi-task benchmarks show that this stability does not worsen general task performance. Consequently, BPGS provides a targeted, scale-robust loss-weighting mechanism that complements existing gradient optimisers instead of attempting to replace them.

The study has clear limits. Our empirical evaluation is restricted to a focused setup that includes NYUv2, two tabular benchmarks, and a synthetic scaling experiment. Extending the validation to additional dense-prediction benchmarks is an immediate next step. The primary benchmark results reflect three random seeds (with 10-seed evaluations reserved for scaling analyses), preventing definitive fine-grained ranking claims among closely performing baselines. As BPGS regulates scalar loss weights rather than gradient trajectories, gradient surgery methods like PCGrad and loss-weighting baselines like Kendall retain an advantage in regimes where directional gradient conflict dominates over scale imbalance. Combining BPGS with such methods is untested future work.

From an optimisation standpoint, bounding the uncertainty parameters introduces a theoretical risk of gradient saturation near the chart boundaries. As reported in Section~\ref{sec:boundedness}, our experiments show an empirical maximum of $|z_i|/\tau_T = 0.93$ without observing saturation, but this operating margin could change under different architectures or learning rates. Furthermore, our theoretical analysis establishes batch-conditional boundedness and normalised-weight invariance but does not extend to formal convergence rates or asymptotic optimality. The diagnostic sensitivity analyses cover only batch sizes 4--16 on NYUv2, and our dense-prediction evaluation includes Nash-MTL as the primary representative baseline, leaving benchmarks against other multi-objective and meta-weighting methods (e.g., CAGrad, IMTL-G, FAMO, and Auto-Lambda \cite{liu2022autolambda}) to future empirical studies.

The runtime study quantifies cost only on the controlled NYUv2 subset and one GPU setting. Potential societal impact is indirect. BPGS may improve multi-output models used in different settings by providing steadier task weights, but downstream deployments still inherit the data, labeling, and evaluation failures of their tasks. In safety-sensitive settings, improving one task while weakening another can be obscured by aggregate scores. We therefore report per-task metrics and state where BPGS is not the best method.

\section{Conclusion}

We propose BPGS, a bounded, batch-anchored reparameterisation of uncertainty weighting designed for multi-task learning under disparate loss scales. Theoretically, BPGS guarantees that normalised task weights remain invariant to uniform loss rescaling under non-degenerate loss scales. Empirically, macro scores shift by less than 0.01 across the $\times1$--$\times1000$ rescaling range, demonstrating substantially greater stability than existing baselines under severe scale disparities. Across multiple benchmarks, BPGS achieves the lowest aggregate multi-task loss and yields the best depth metrics and lowest total loss on the NYUv2 dataset, while other methods lead on segmentation and normals; on tabular tasks such as Yeast and RF1, it yields leading or highly competitive performance.

We do not claim a universal ranking. BPGS is designed solely for regimes where loss-scale disparity, rather than directional gradient interference, is the major source of optimisation instability. As such, it complements rather than replaces gradient-space methods. Promising directions for future research include evaluating BPGS across a broader dense-prediction suite and exploring hybrid frameworks that combine bounded loss weighting with conflict-aware optimisers to simultaneously address scale imbalance and gradient conflict.

\bibliographystyle{unsrtnat}
\bibliography{references}

@article{caruana1997multitask,
  title={Multitask Learning},
  author={Caruana, Rich},
  journal={Machine Learning},
  volume={28},
  number={1},
  pages={41--75},
  year={1997},
  doi={10.1023/A:1007379606734}
}

@inproceedings{kendall2018,
  title={Multi-Task Learning Using Uncertainty to Weigh Losses for Scene Geometry and Semantics},
  author={Kendall, Alex and Gal, Yarin and Cipolla, Roberto},
  booktitle={Proceedings of the IEEE Conference on Computer Vision and Pattern Recognition},
  pages={7482--7491},
  year={2018}
}

@inproceedings{guo2018dtp,
  title={Dynamic Task Prioritization for Multitask Learning},
  author={Guo, Michelle and Haque, Albert and Huang, De-An and Yeung, Serena and Fei-Fei, Li},
  booktitle={Proceedings of the European Conference on Computer Vision (ECCV)},
  pages={282--299},
  year={2018}
}

@inproceedings{yu2020pcgrad,
  title={Gradient surgery for multi-task learning},
  author={Yu, Tianhe and Kumar, Saurabh and Gupta, Abhishek and Levine, Sergey and Hausman, Karol and Finn, Chelsea},
  booktitle={Advances in Neural Information Processing Systems},
  year={2020}
}

@inproceedings{liu2019mtan,
  title={End-To-End Multi-Task Learning With Attention},
  author={Liu, Shikun and Johns, Edward and Davison, Andrew J.},
  booktitle={Proceedings of the IEEE/CVF Conference on Computer Vision and Pattern Recognition},
  year={2019}
}

@inproceedings{liu2021cagrad,
  title={Conflict-Averse Gradient Descent for Multi-task Learning},
  author={Liu, Bo and Liu, Xingchao and Jin, Xiaojie and Stone, Peter and Liu, Qiang},
  booktitle={Advances in Neural Information Processing Systems},
  year={2021}
}

@inproceedings{silberman2012nyuv2,
  title={Indoor segmentation and support inference from {RGBD} images},
  author={Silberman, Nathan and Hoiem, Derek and Kohli, Pushmeet and Fergus, Rob},
  booktitle={European Conference on Computer Vision},
  year={2012}
}

@inproceedings{elisseeff2001yeast,
  title={A kernel method for multi-labelled classification},
  author={Elisseeff, Andr{\'e} and Weston, Jason},
  booktitle={Advances in Neural Information Processing Systems},
  pages={681--687},
  year={2001}
}

@article{spyromitros2016mtr,
  title={Multi-target regression via input space expansion: treating targets as inputs},
  author={Spyromitros-Xioufis, Eleftherios and Tsoumakas, Grigorios and Groves, William and Vlahavas, Ioannis},
  journal={Machine Learning},
  volume={104},
  number={1},
  pages={55--98},
  year={2016},
  doi={10.1007/s10994-016-5546-z}
}

@inproceedings{chen2018gradnorm,
  title={GradNorm: Gradient Normalization for Adaptive Loss Balancing in Deep Multitask Networks},
  author={Chen, Zhao and Badrinarayanan, Vijay and Lee, Chen-Yu and Rabinovich, Andrew},
  booktitle={Proceedings of the 35th International Conference on Machine Learning},
  series={Proceedings of Machine Learning Research},
  volume={80},
  pages={794--803},
  year={2018}
}

@inproceedings{sener2018multiobjective,
  title={Multi-Task Learning as Multi-Objective Optimization},
  author={Sener, Ozan and Koltun, Vladlen},
  booktitle={Advances in Neural Information Processing Systems},
  year={2018}
}

@inproceedings{navon2022nash,
  title={Multi-Task Learning as a Bargaining Game},
  author={Navon, Aviv and Shamsian, Aviv and Achituve, Idan and Maron, Haggai and Kawaguchi, Kenji and Chechik, Gal and Fetaya, Ethan},
  booktitle={Proceedings of the 39th International Conference on Machine Learning},
  series={Proceedings of Machine Learning Research},
  volume={162},
  pages={16428--16446},
  year={2022}
}

@inproceedings{liu2021imtl,
  title={{Towards Impartial Multi-Task Learning}},
  author={Liu, Liyang and Li, Yi and Kuang, Zhanghui and Xue, Jing-Hao and Chen, Yimin and Yang, Wenming and Liao, Qingmin and Zhang, Wayne},
  booktitle={International Conference on Learning Representations},
  year={2021},
  url={https://openreview.net/forum?id=IMPnRXEWpvr}
}

@inproceedings{liu2023famo,
  title={{FAMO}: Fast Adaptive Multitask Optimization},
  author={Liu, Bo and Feng, Yihao and Stone, Peter and Liu, Qiang},
  booktitle={Advances in Neural Information Processing Systems},
  year={2023}
}

@article{liu2022autolambda,
  title={Auto-Lambda: Disentangling Dynamic Task Relationships},
  author={Liu, Shikun and James, Stephen and Davison, Andrew J. and Johns, Edward},
  journal={Transactions on Machine Learning Research},
  year={2022}
}

@article{desideri2012mgda,
  title={Multiple-gradient descent algorithm ({MGDA}) for multiobjective optimization},
  author={D{\'e}sid{\'e}ri, Jean-Antoine},
  journal={Comptes Rendus Mathematique},
  volume={350},
  number={5--6},
  pages={313--318},
  year={2012},
  doi={10.1016/j.crma.2012.03.014}
}

\appendix
\section{Reproducibility Details}
\label{app:reproducibility}
Unless stated otherwise, all experiments aggregate results over three random seeds (42, 43, and 44), except the synthetic scale-stress benchmark, which evaluates over 10 seeds (42--51). For real-world benchmarks (NYUv2, Yeast, and RF1), metrics are reported at the final training epoch across all methods to ensure uniform evaluation without checkpoint selection bias.

\paragraph{Optimization and numerical safeguards.}
All neural network models are trained using AdamW with linear learning-rate warmup and cosine decay. In BPGS, the batch log-loss statistics $\mu(L)$ and $\bar\varsigma(L)$ as well as the normalized weights $\alpha_i$ are detached during network updates, while task losses are detached during uncertainty updates with fixed gradient scale $g_\theta=100$. Numerical stability floors are set to $\varepsilon_{\log}=10^{-8}$ for logarithmic arguments and $\varepsilon_{\mathrm{std}}=10^{-4}$ for standard deviation. Gradient clipping is set to 10.0 for BPGS and 1.0 for Kendall.

\paragraph{Main NYUv2 benchmark.}
The standard NYUv2 benchmark contains 795 training and 654 validation images ($288\times 384$ resolution) evaluated on a 3-task SegNet architecture ($\approx 18.9\text{M}$ parameters): 13-class semantic segmentation (cross-entropy loss, ignore index 255), depth estimation (masked L1 loss), and surface-normal prediction (cosine distance loss). Training runs for 120 epochs with batch size 8, initial learning rate $5\times 10^{-4}$ decaying with cosine schedule to $10^{-6}$ (100 warmup steps), weight decay $10^{-4}$, gradient clipping 10.0, deterministic CPU data augmentation, and mixed precision.

\paragraph{NYUv2 ablation.}
The ablation uses a fixed 50\% training subset (398 images, stratified by majority semantic class and depth quartile using seed~11) with the same SegNet architecture and validation split. Models are trained for 60 epochs with batch size 8, initial learning rate $5\times 10^{-4}$ (decaying to $10^{-6}$ after 100 warmup steps), weight decay $10^{-4}$, and mixed precision.

\paragraph{Yeast and RF1.}
Yeast and RF1 use 4-layer shared-trunk MLPs (hidden dimension 256, dropout 0.1). Yeast consists of 1{,}500 training and 917 validation examples for 14-label classification under binary cross-entropy loss; models are trained for up to 120 epochs with batch size 128, learning rate $5\times 10^{-4}$ decaying to $10^{-5}$ (200 warmup steps), weight decay $10^{-4}$, gradient clipping 10.0, and early stopping patience of 15 epochs on validation loss. RF1 consists of 4{,}108 training and 5{,}017 validation examples for 8-target regression under MSE loss, trained with batch size 256 and identical optimizer settings.

\paragraph{Synthetic stress benchmarks.}
Synthetic scale-stress uses $T=4$ tasks over 10 seeds (42--51); pure loss-rescaling uses $T=5$ tasks over seeds 42, 43, and 44; and heterogeneous mixed-stress uses $T=8$ tasks across three regimes: \emph{clean} ($\rho=0.45, \eta=0.05$), \emph{noisy} ($\rho=0.15, \eta=0.14$), and \emph{conflict} ($\rho=-0.35, \eta=0.10$).

\paragraph{Baseline methods.}
The \texttt{GradNormProxy} baseline implements a gradient-norm balancing proxy inspired by GradNorm. The UWSO baseline implements an analytical inverse-loss weighting rule with fixed softmax temperature $T_{\mathrm{temp}}=2.0$.

\paragraph{Additional NYUv2 and diagnostic studies.}
The stop-gradient, batch-size, first-batch, and computational overhead studies all use the 50\% NYUv2 subset for 60 epochs over three seeds (42, 43, 44). For Nash-MTL, training follows the 120-epoch NYUv2 protocol with batch size 8, AdamW learning rate $5\times10^{-4}$ (cosine decay to $10^{-6}$, 100 warmup steps), weight decay $10^{-4}$, mixed precision, and gradient clipping 1.0. The normalization ablation evaluates on the pure loss-rescaling grid over seeds 42, 123, and 999.

\paragraph{Reproducibility disclosures.}
All experiments use standard public benchmarks cited in the main text. Complete source code, configuration files, and reproduction scripts are publicly available at \url{https://github.com/neryva-lab/spectra} and included in the supplementary material. Computational overhead measurements are systematically reported for the controlled benchmark setting in Appendix~\ref{app:overhead}.

\section{Additional Stress Results}
\label{app:norm_ablation}
\begin{table}[htbp]
\centering
\caption{Synthetic scale stress: Macro Score across scale factors. Mean $\pm$ std over 10 seeds.}
\label{tab:stress_scale}
\resizebox{\columnwidth}{!}{%
\begin{tabular}{l c c c c}
\toprule
Method & $\times{1}$ & $\times{10}$ & $\times{100}$ & $\times{1000}$ \\
\midrule
Kendall & $0.738 \pm 0.017$ & $0.607 \pm 0.038$ & $0.522 \pm 0.030$ & $0.499 \pm 0.025$ \\
UWSO & $0.690 \pm 0.034$ & $0.604 \pm 0.040$ & $0.525 \pm 0.035$ & $0.494 \pm 0.024$ \\
BPGS & $\mathbf{0.743 \pm 0.018}$ & $\mathbf{0.616 \pm 0.039}$ & $\mathbf{0.547 \pm 0.032}$ & $\mathbf{0.543 \pm 0.032}$ \\
\bottomrule
\end{tabular}
}
\end{table}

\begin{table}[htbp]
\centering
\caption{Relative Macro Score change (\%) from $\times 1 \to \times 1000$. Positive values indicate degradation; negative values indicate improvement.}
\label{tab:degradation}
\begin{tabular}{l c c}
\toprule
Method & Scale Stress & Loss Rescaling \\
\midrule
Kendall & $32.3\%$ & $18.4\%$ \\
UWSO & $28.5\%$ & $\mathbf{-9.5\%}$ \\
BPGS & $\mathbf{27.0\%}$ & $0.0\%$ \\
\bottomrule
\end{tabular}
\end{table}

\begin{table}[htbp]
\centering
\caption{Heterogeneous mixed stress: Macro Score and Worst-Task Score across regimes. Mean $\pm$ std over 3 seeds.}
\label{tab:stress_heterogeneous}
\resizebox{\columnwidth}{!}{%
\begin{tabular}{l c c c c c c}
\toprule
Method & \multicolumn{2}{c}{Clean} & \multicolumn{2}{c}{Noisy} & \multicolumn{2}{c}{Conflict} \\
& Macro & Worst & Macro & Worst & Macro & Worst \\
\midrule
Kendall & $0.687 \pm 0.010$ & $0.267 \pm 0.056$ & $\mathbf{0.659 \pm 0.007}$ & $\mathbf{0.242 \pm 0.045}$ & $\mathbf{0.664 \pm 0.012}$ & $\mathbf{0.211 \pm 0.041}$ \\
UWSO & $0.437 \pm 0.028$ & $0.000 \pm 0.000$ & $0.438 \pm 0.020$ & $0.000 \pm 0.000$ & $0.433 \pm 0.013$ & $0.000 \pm 0.000$ \\
PCGrad & $\mathbf{0.688 \pm 0.016}$ & $0.274 \pm 0.055$ & $0.656 \pm 0.007$ & $0.241 \pm 0.046$ & $0.660 \pm 0.011$ & $0.205 \pm 0.048$ \\
BPGS & $0.679 \pm 0.017$ & $\mathbf{0.301 \pm 0.052}$ & $0.641 \pm 0.013$ & $0.176 \pm 0.076$ & $0.652 \pm 0.012$ & $0.198 \pm 0.063$ \\
\bottomrule
\end{tabular}
}
\end{table}
\begin{table}[htbp]
\centering
\caption{Normalization ablation on pure loss rescaling. Macro score (mean $\pm$ std over seeds 42, 123, and 999). L1-normalizing Kendall improves the largest-scale result but does not reproduce the BPGS scale sensitivity.}
\label{tab:norm_ablation}
\small
\begin{tabular}{lccc}
\toprule
Scale & BPGS & Kendall & Kendall+L1 \\
\midrule
$\times1$    & $0.789 \pm 0.018$ & $0.788 \pm 0.014$ & $0.794 \pm 0.017$ \\
$\times10$   & $0.791 \pm 0.015$ & $0.781 \pm 0.023$ & $0.780 \pm 0.022$ \\
$\times100$  & $0.788 \pm 0.016$ & $0.734 \pm 0.037$ & $0.742 \pm 0.029$ \\
$\times1000$ & $0.785 \pm 0.017$ & $0.658 \pm 0.042$ & $0.689 \pm 0.034$ \\
\midrule
$\Delta(\times1\!\to\!\times1000)$ & $-0.004$ & $-0.130$ & $-0.105$ \\
\bottomrule
\end{tabular}
\end{table}

\begin{figure*}[t]
\centering
\includegraphics[width=\textwidth]{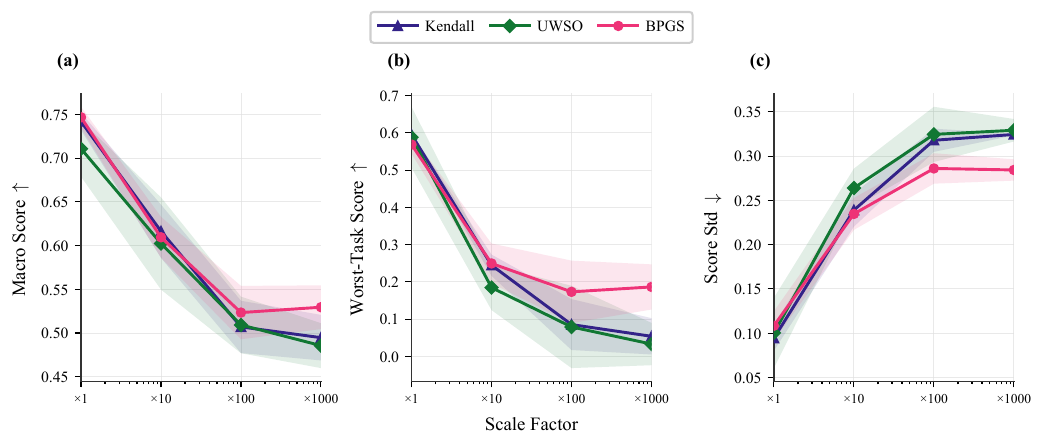}
\caption{Synthetic scale-stress curves. As the scale factor increases, all methods degrade, but BPGS retains the strongest macro score at every tested factor and tracks the strongest worst-task score at the larger stress levels.}
\label{fig:appendix_scale_stress}
\end{figure*}

\begin{figure*}[htbp]
\centering
\includegraphics[width=\textwidth]{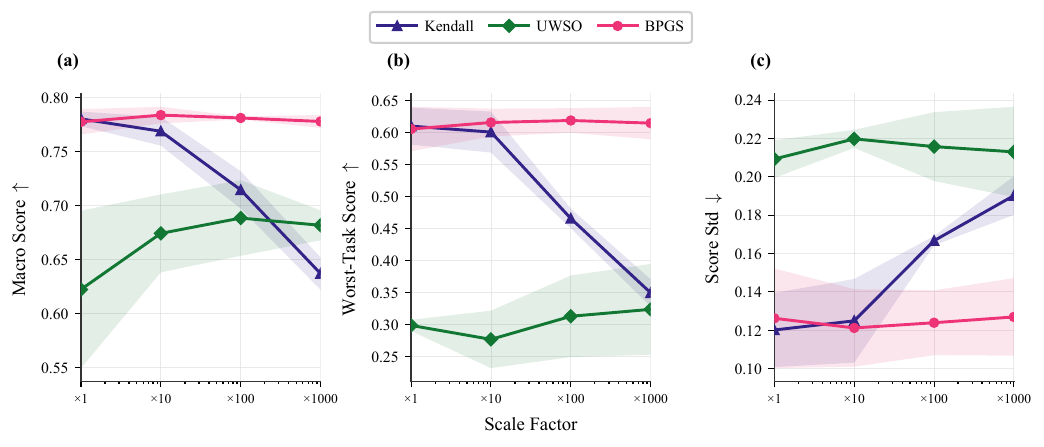}
\caption{Pure loss-rescaling curves. BPGS remains nearly invariant across the entire rescaling range, while Kendall degrades sharply at larger factors and UWSO improves only from a substantially weaker starting point.}
\label{fig:appendix_rescaling}
\end{figure*}

\begin{figure*}[htbp]
\centering
\includegraphics[width=\textwidth]{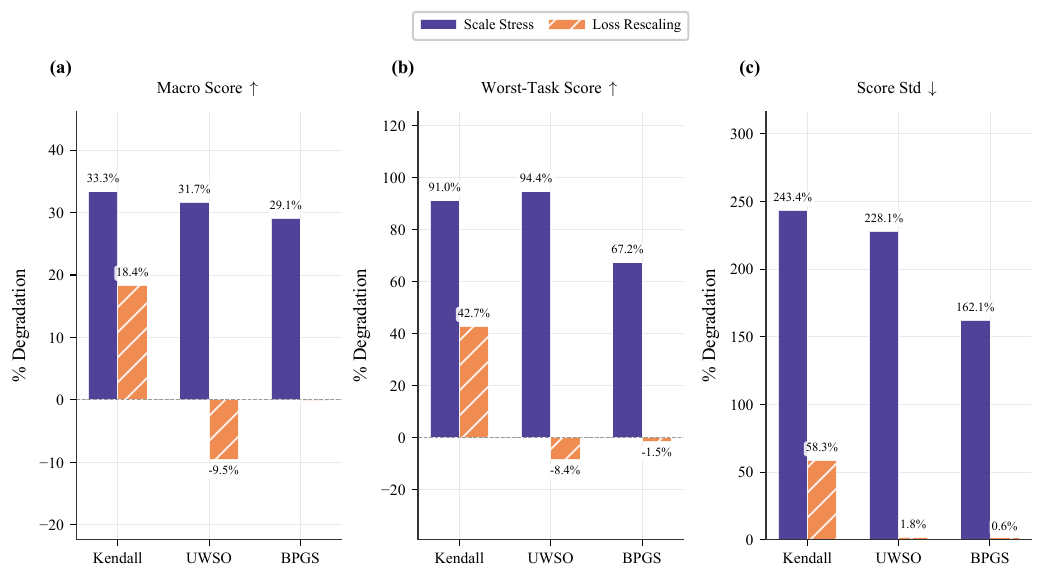}
\caption{Relative degradation comparison between synthetic scale stress and pure loss rescaling. The figure shows that BPGS has the smallest macro-score degradation under scale stress and essentially no macro-score degradation under pure rescaling.}
\label{fig:appendix_degradation}
\end{figure*}

\begin{figure*}[t]
\centering
\includegraphics[width=\textwidth]{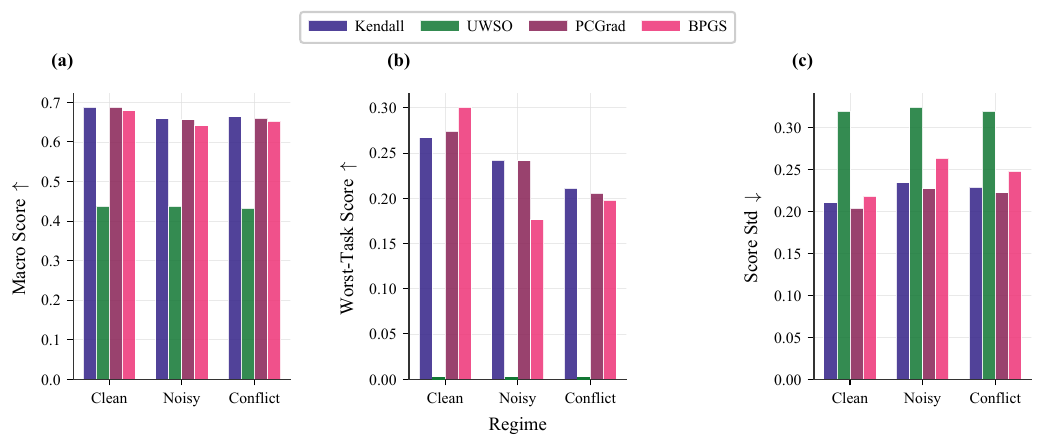}
\caption{Heterogeneous mixed-stress summary. BPGS is not the best method in every regime, but it achieves the strongest worst-task score in the clean regime and remains well above the degenerate UWSO worst-task score across all regimes.}
\label{fig:appendix_heterogeneous}
\end{figure*}

\section{Ablation and NYUv2 Figures}
Nash-MTL's per-step coefficients varied substantially across batches in all three runs. This
observation is descriptive only: the present logs do not distinguish solver variation from an
intrinsic property of the method on this benchmark.

\begin{figure*}[t]
\centering
\includegraphics[width=\textwidth]{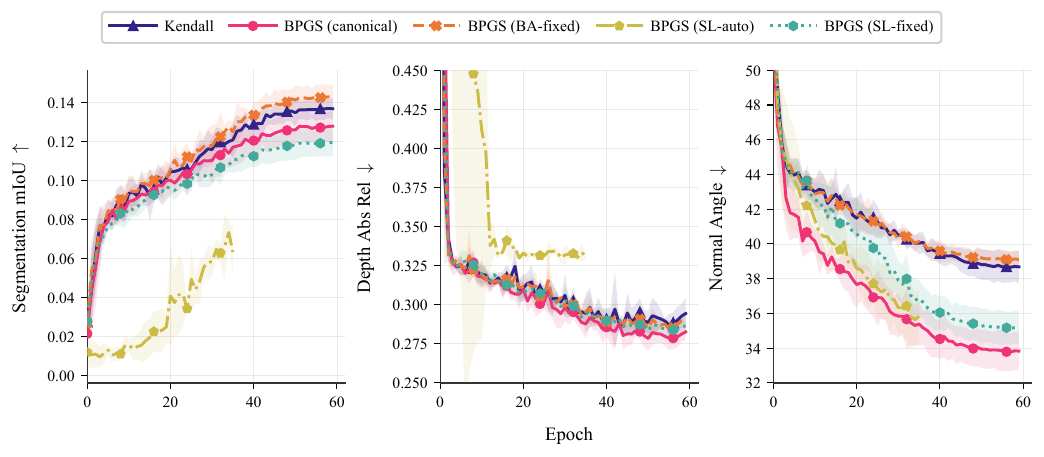}
\caption{NYUv2 ablation validation curves by task. The stateless auto-calibrated variant fails on segmentation and remains weak on depth, while the batch-aware variants maintain stronger trajectories throughout training.}
\label{fig:appendix_ablation_val_curves}
\end{figure*}

\begin{figure*}[t]
\centering
\includegraphics[width=\textwidth]{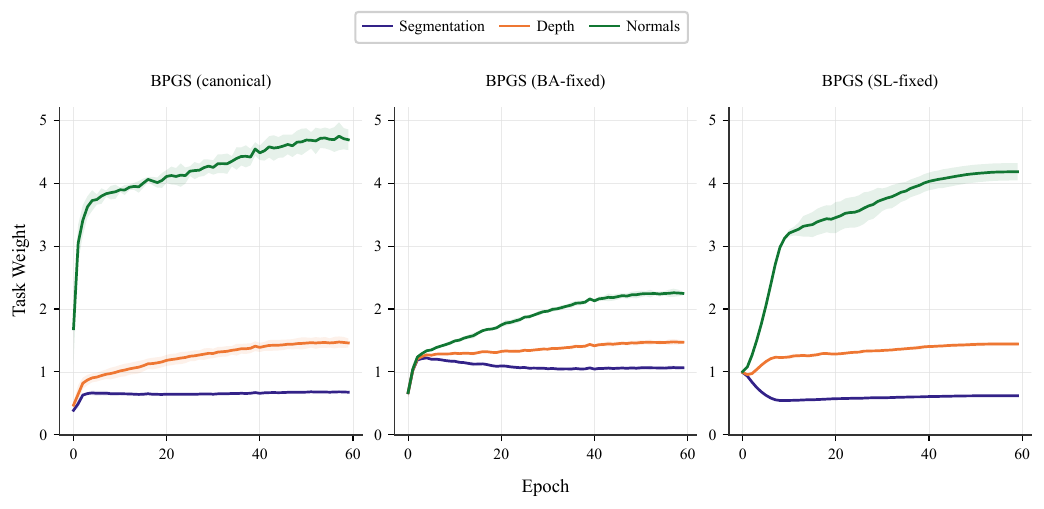}
\caption{Task-weight dynamics for representative stable BPGS variants on the NYUv2 ablation. The batch-aware variants keep task weights in distinct but controlled regimes, whereas the stateless fixed variant shows a different long-run allocation pattern, especially for the surface-normal task.}
\label{fig:appendix_ablation_weights}
\end{figure*}

\begin{figure*}[t]
\centering
\includegraphics[width=\textwidth]{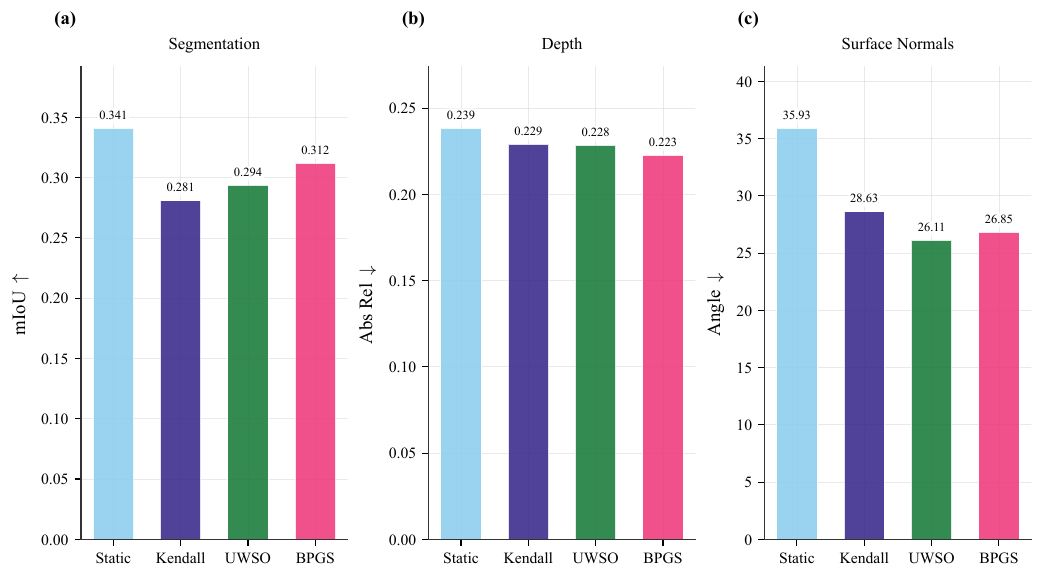}
\caption{NYUv2 per-task performance comparison. The figure makes the trade-off pattern visually explicit: Static is strongest on segmentation, UWSO is strongest on surface normals, and BPGS is strongest on both depth metrics.}
\label{fig:appendix_nyuv2_bars}
\end{figure*}

\begin{figure*}[t]
\centering
\includegraphics[width=\textwidth]{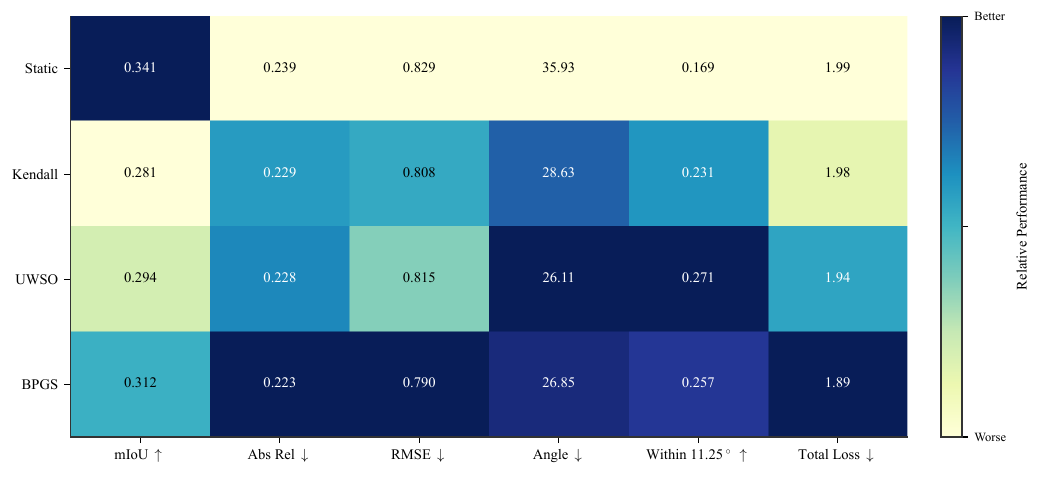}
\caption{NYUv2 performance heatmap across reported metrics. Darker cells indicate stronger relative performance after accounting for metric direction. BPGS concentrates its gains on depth and total loss, while Static and UWSO remain stronger on some other tasks.}
\label{fig:appendix_nyuv2_heatmap}
\end{figure*}

\begin{figure*}[t]
\centering
\includegraphics[width=\textwidth]{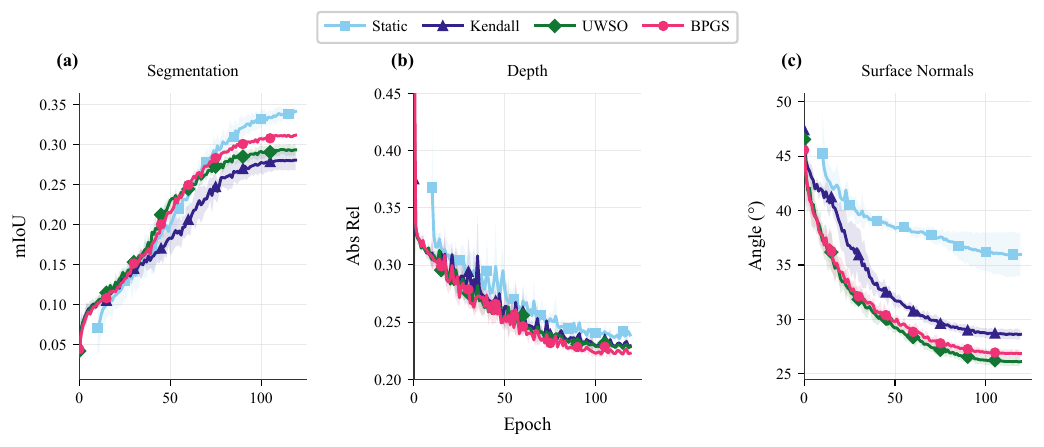}
\caption{NYUv2 validation training curves by task. BPGS tracks the strongest depth trajectory and remains competitive on segmentation and surface normals over the full 120-epoch budget.}
\label{fig:appendix_nyuv2_curves}
\end{figure*}

\begin{figure}[t]
\centering
\includegraphics[width=\columnwidth]{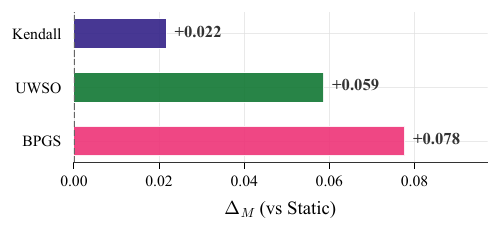}
\caption{NYUv2 $\Delta_M$ comparison against the Static baseline. BPGS attains the strongest aggregate improvement among the adaptive methods evaluated in the main benchmark.}
\label{fig:appendix_nyuv2_delta}
\end{figure}

\section{Real-Data Training Curves}

\begin{figure*}[t]
\centering
\includegraphics[width=\textwidth]{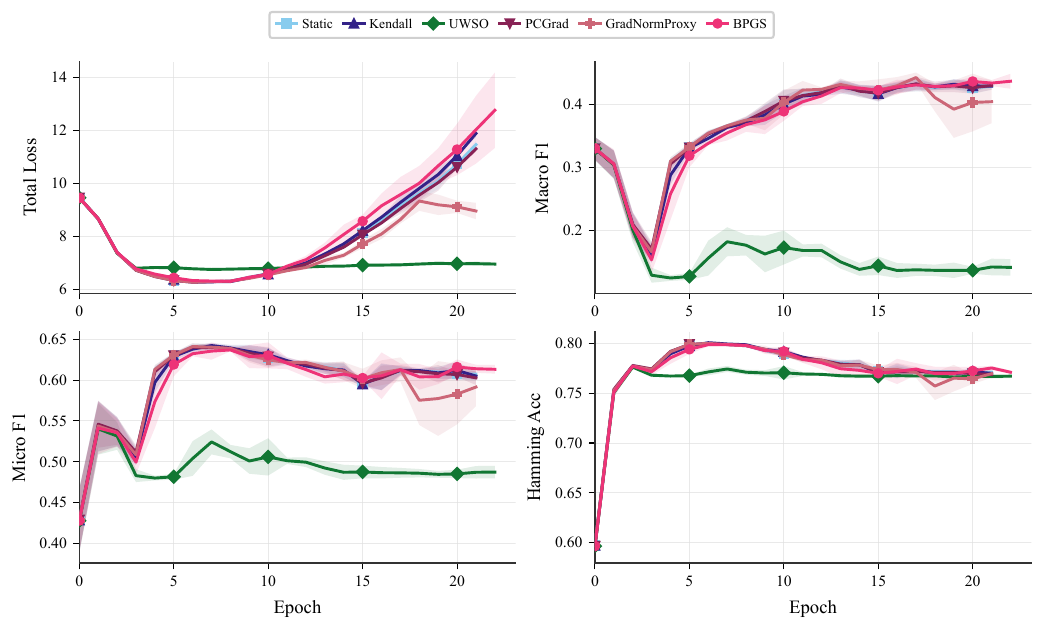}
\caption{Yeast training curves. BPGS tracks the strongest final micro-F1 and macro-F1 among the compared methods, while UWSO remains substantially weaker throughout training.}
\label{fig:appendix_yeast_curves}
\end{figure*}

\begin{figure*}[t]
\centering
\includegraphics[width=\textwidth]{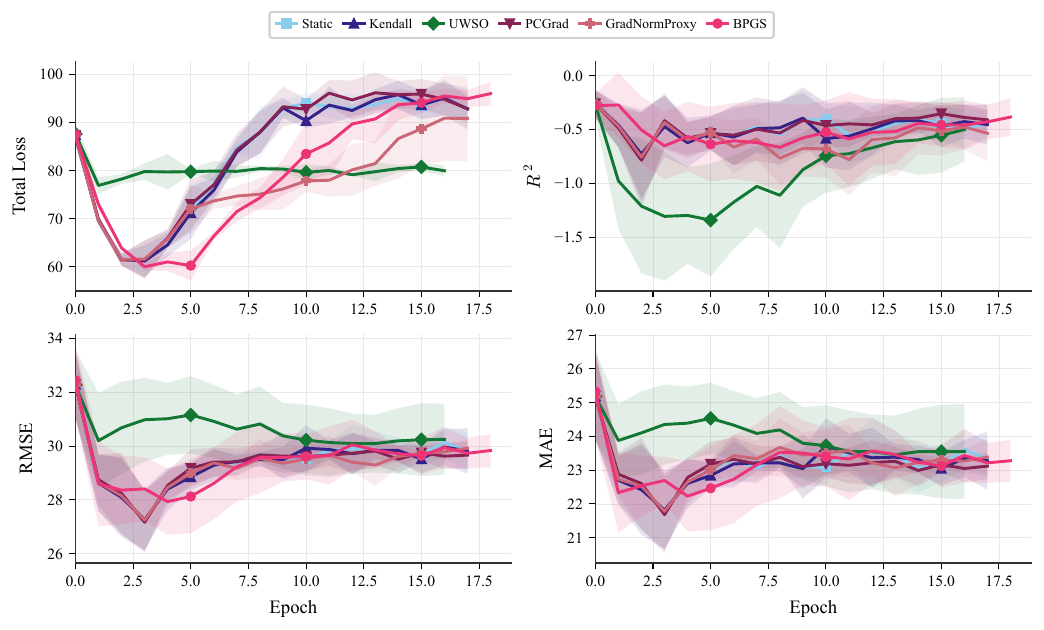}
\caption{RF1 training curves. The methods remain closer than in the synthetic stress studies, which is consistent with the paper's claim that RF1 supports competitiveness rather than clear superiority.}
\label{fig:appendix_rf1_curves}
\end{figure*}

\section{Sensitivity Analyses}
\label{app:sensitivity}
\begin{table}[t]
\centering
\caption{Stop-gradient ablation on the NYUv2 50\% subset (three seeds, 60 epochs). Removing the stop-gradient changes the uncertainty dynamics and lowers segmentation mIoU.}
\label{tab:stop_gradient}
\small
\begin{tabular}{lcc}
\toprule
Metric & Split on & Split off \\
\midrule
$\theta_{\max}$, epoch 0 & $3.082 \pm 0.017$ & $3.109 \pm 0.011$ \\
$\theta_{\max}$, final & $2.438 \pm 0.014$ & $3.993 \pm 0.192$ \\
Segmentation mIoU & $0.1983$ & $0.1897$ \\
Depth RMSE & $0.8433$ & $0.8544$ \\
Total loss & $2.1278$ & $2.1419$ \\
\bottomrule
\end{tabular}
\end{table}

\begin{table}[t]
\centering
\caption{Batch-size sensitivity of canonical BPGS on the NYUv2 50\% subset (three seeds per batch size, 60 epochs). Values are seed-averaged final metrics.}
\label{tab:batch_size}
\small
\begin{tabular}{lrrrr}
\toprule
Metric & bs=4 & bs=8 & bs=16 & CV \\
\midrule
Total loss & 2.1468 & 2.1226 & 2.1507 & 0.58\% \\
Depth Abs Rel & 0.2411 & 0.2421 & 0.2451 & 0.70\% \\
Depth RMSE & 0.8643 & 0.8403 & 0.8490 & 1.17\% \\
Normal angle & 29.400 & 29.223 & 29.754 & 0.75\% \\
Segmentation mIoU & 0.1955 & 0.1932 & 0.1810 & 3.35\% \\
$\theta_{\max}$, final & 2.396 & 2.441 & 2.589 & -- \\
\bottomrule
\end{tabular}
\end{table}

\begin{table}[t]
\centering
\caption{First-batch calibration sensitivity on the NYUv2 50\% subset. Only the loader seed changes; the model seed is fixed.}
\label{tab:first_batch}
\small
\begin{tabular}{lrrrr}
\toprule
Metric & 1001 & 2001 & 3001 & CV \\
\midrule
Total loss & 2.1145 & 2.1040 & 2.1090 & 0.20\% \\
Depth Abs Rel & 0.2385 & 0.2392 & 0.2411 & 0.46\% \\
Depth RMSE & 0.8348 & 0.8302 & 0.8325 & 0.23\% \\
Normal angle & 28.9205 & 28.8975 & 29.1069 & 0.32\% \\
Segmentation mIoU & 0.1912 & 0.1866 & 0.1952 & 1.84\% \\
$\theta_{\max}$, final & 2.4451 & 2.4382 & 2.4380 & 0.3\% spread \\
\bottomrule
\end{tabular}
\end{table}

The stop-gradient study changes only the coupling of the uncertainty update. The batch-size and
first-batch studies are controlled diagnostics on the NYUv2 subset, not claims about all dataset
sizes, batch regimes, or initialization procedures.

\section{Computational Overhead}
\label{app:overhead}
\begin{table}[t]
\centering
\caption{Computational overhead on the NYUv2 50\% subset (three seeds, 60 epochs).}
\label{tab:overhead}
\small
\begin{tabular}{lrrr}
\toprule
Quantity & BPGS & Kendall & Difference \\
\midrule
Per-epoch time (s) & 40.979 & 40.916 & +0.15\% \\
Peak GPU memory (MB) & 3368.36 & 3339.34 & +0.87\% \\
Total time, 60 epochs (s) & 2461.6 & 2457.9 & +3.7 \\
Trainable parameters & 18,869,143 & 18,869,140 & +3 \\
\bottomrule
\end{tabular}
\end{table}

The measurements report the controlled BPGS--Kendall comparison on a single GPU setting. They
quantify the added batch-statistics computation and uncertainty pass, but do not establish a
hardware-independent efficiency guarantee.

\section{Bounded-Chart Behavior}
\label{app:chart_behavior}
\begin{table}[t]
\centering
\caption{Closest observed approach to the bounded-chart radius across logged main-paper runs. The maximum occurs at initialization; no run reaches the boundary.}
\label{tab:saturation}
\small
\begin{tabular}{lrrr}
\toprule
Dataset & $T$ & max $|z_i|/\tau_T$ & final cap \\
\midrule
NYUv2 & 3 & 0.9327 & $\leq 0.84$ \\
Yeast & 14 & 0.7017 & $\leq 0.69$ \\
RF1 & 8 & 0.8937 & $\leq 0.89$ \\
\bottomrule
\end{tabular}
\end{table}

We reconstruct $z_i=(s_i-\mu)/\bar\varsigma$ from the logged batch statistics and log-variances.
At the closest observed point, $\sigma(\theta_i)(1-\sigma(\theta_i))\approx0.033$, about 13\%
of its maximum 0.25; by the final epochs the corresponding value is approximately 0.074. These
measurements show an empirical margin from saturation in the logged runs, not a general guarantee.

\clearpage

\end{document}